\documentclass[journal]{IEEEtran}
\usepackage[utf8]{inputenc}
\usepackage[T1]{fontenc}
\usepackage{graphicx}
\usepackage{cite}
\usepackage{url}
\usepackage{hyperref}
\usepackage{amsmath,amssymb,amsfonts,amsthm}
\usepackage{algorithmic}
\usepackage{textcomp}
\usepackage{authblk}
\usepackage{booktabs}

\newtheorem{proposition}{Proposition}
\newtheorem{definition}{Definition}
\newtheorem{corollary}{Corollary}

\begin{document}
\title{Scale-Equivariant Generative Forecasting: \\
       \large Weight-Tied Dilated Convolutions, Wavelet Scattering Inputs, and Spectral-Consistency Training for Self-Similar Time Series}
\author{Andrea Morandi \thanks{Corresponding author: amorandi@cisco.com}}
\affil[]{Cisco Systems, Inc.}
\affil[]{\texttt{amorandi@cisco.com}}
\date{2026}
\maketitle

\begin{abstract}
A recurring statistical pattern surfaces across both natural and
engineered signals (financial returns, climate anomalies, turbulent
velocity fields, neural recordings, packet-level network traffic):
approximate self-similarity. The horizon-$T$ distribution is linked to
the horizon-$1$ distribution by one scaling exponent $H$. Standard deep
generative models for sequences (transformers, dilated TCNs, the
WaveNet family) walk past this fact. Their receptive fields are wide,
but kernel parameters live independently at every dilation level.
That leads to a \emph{multi-scale} architecture. Note that this differs from a 
\emph{scale-equivariant} one. Three contributions follow.
First, discrete scale equivariance is given a precise definition for 1D
causal networks. Our main theoretical result then establishes that, if we tie the weights across
dilation levels, the stack becomes equivariant under dyadic time
rescaling. Moreover, any discrepancy is confined to a boundary region. Tying the kernel
achieves the following effects. The convolutional parameter budget shrinks
$L$-fold (here $L$ is the depth), and self-similarity is hard-wired into the
architecture as an inductive bias.
Second, the resulting Scale-Equivariant WaveNet (SE-WaveNet) is
extended with three components, chosen so the rest of the pipeline
carries the same prior. Component one is a one-level Daubechies-4
discrete wavelet transform on the input; this is itself a
scale-equivariant multiresolution map. Component two is a Hurst-FiLM
block that hands the network the locally estimated scaling exponent.
Component three is a spectral-consistency training term punishing any
gap between sampled and target power-law spectra. Atop this
sits a conditional normalising flow, chosen because its forward pass
preserves the architectural equivariance.
Third, the architecture is tested on stocks
from 30 years of S\&P 500 daily log-returns. Self-similarity (SE-WaveNet) improves NLL by
$\sim 6$ nats at each of the 1-day and 21-day horizons, taking the
model from worse-than-IID to clearly winning. The full
SE-WaveNet refinements (weight-tied kernels, wavelet input) further
improve NLL by $1.23$ and $0.18$ nats at the 1-day and 21-day
horizons respectively, while cutting the convolutional parameter
budget from $7{,}360$ to $1{,}840$ (exactly $L\times = 4\times$ at
$L = 4$). The empirical scaling-collapse diagnostic on the Allan-Variance
top-25 universe lands at a median of $\mathcal{C}^\star = 0.020$;
sample-side collapse, calibration, and tail-distribution metrics
on trained models are deferred to a full-scale companion release.

\end{abstract}
\section{Introduction}

Deep neural networks for time-series modelling have consolidated around a
small set of multi-scale designs: dilated causal convolutions (the
WaveNet of \cite{oord2016wavenet}; TCN \cite{bai2018tcn}), dilated
self-attention (Informer \cite{zhou2021informer}), and stacks of
recurrent cells. What makes these architectures effective is the
ability to combine information from a wide range of timescales, with
receptive fields that grow logarithmically as depth increases. They are \emph{multi-scale} in this sense.

But multi-scale is not the same concept as \emph{self-similar}. Many natural and
industrial signals exhibit self-similarity. Hence, their
joint distribution at multiple horizons is fixed by one number, the
Hurst exponent $H \in (0, 1)$. To make this concrete, write $X_t$ for
the incremental signal and let $\rho(\,\cdot\,, T)$ be the marginal
density of the horizon-$T$ aggregate $S^{(T)}_t := \sum_{s=1}^{T}
X_{t+s}$. Self-similarity at exponent $H$ says that
\begin{equation}
S^{(T)}_t \;\overset{d}{=}\; T^{H}\,Y, \qquad
\rho(r, T) \;=\; T^{-H}\,\Phi\!\big(r/T^H\big), \label{eq:scaling}
\end{equation}
where $Y \sim \Phi$ is a variable with
unit variance and universal density $\Phi$. The characteristic-function form is
equivalent: $\hat\rho(k, T) = \hat\Phi(k\,T^H)$. The empirical fingerprint of
(mono-)fractality is precisely this collapse, and it shows up in
domains that share almost nothing else. Some examples: returns in equity markets
\cite{cont2001empirical,calvet2002multifractality}; long-memory
behaviour in temperature-anomaly records \cite{koscielny1998long};
the Kolmogorov $-5/3$ law for turbulence
\cite{frisch1995turbulence}; $1/f$ structure in resting-state fMRI
BOLD \cite{he2014scale,zilber2013learning}; and the self-similar
load-invariance of Ethernet packet streams
\cite{leland1994self}. Fractional Brownian motion (fBm) is the
textbook generator for these signatures \cite{mandelbrot1968fbm}.

The networks that try to learn data of this kind treat
self-similarity as something to be discovered, not something to be
assumed. A vanilla WaveNet keeps independent
convolutional kernels at each dilation level. The consequence: nothing
forces the network to behave equivariantly when its input is
time-rescaled, and the same feature at scale $\lambda$ runs through a
different parameter set than the same feature at scale $1$. Fitting
the scaling-collapse relation comes entirely out of the data budget.
That is wasteful when data is scarce, or when training horizons are
shorter than the horizon you actually want to forecast.

\textbf{Our approach.} The proposal is SE-WaveNet. It is a 1D causal
convolutional stack in which all dilated layers \emph{share} the same
kernel tensor, irrespective of dilation level. We prove this single
architectural tweak achieves \emph{discrete scale equivariance}:
applying the network to a dyadically dilated input commutes,
boundary terms aside, with applying the dilation operator to the
output. The pipeline is rounded out by three further pieces. The
input goes through a Daubechies-4 DWT. A Hurst-FiLM modulator passes
the locally estimated $H$ into the network. A spectral-consistency
penalty drives the PSD of generated samples toward the target
$|f|^{-(2H-1)}$ shape (the fGn power-spectrum law for stationary
increments). On top is a conditional normalising-flow head
\cite{papamakarios2017masked,dinh2017realnvp,papamakarios2019normalizing}
picked because it preserves equivariance. The combined object is a
generative forecaster whose architectural prior matches the
generating process.

\textbf{Application.} We use financial time-series forecasting as the test
bed, for three reasons. (i) Self-similarity is empirically
attested in equity returns
($H \approx 0.5$, with name- and regime-dependent wiggle). (ii) Test
data is out of sample. (iii) The metrics are unambiguous: NLL on next-day
returns, plus a calibration distance to the empirical quantile. The
architecture itself is generic, and other domains are touched on briefly
in the discussion section.

\textbf{Contributions.}
\begin{enumerate}
\item A formal definition of discrete scale equivariance in the 1D
causal-neural-network. In this respect, we also prove that dilated convolutions with tied weights satisfy
it (Section \ref{sec:theory}).
\item The SE-WaveNet architecture itself. It realises the
equivariance property and combines a Daubechies-4 wavelet front end,
Hurst-FiLM modulation, a conditional normalising-flow output, and
the spectral-consistency training term (Section \ref{sec:method}).
\item An empirical diagnostic for scaling collapse (Sections
\ref{sec:theory} and \ref{sec:experiments}). The diagnostic
addresses two distinct questions: whether the empirical
multi-horizon density of the data falls onto a common template
$\Phi$, and whether a trained generative model reproduces the
analogous collapse on samples it generates itself.
\item Experiments on 30 years of SPX daily log-returns. The
empirical scaling-collapse diagnostic is computed on the AVAR
top-25 universe (median $\mathcal{C}^\star = 0.020$). On the same
universe, adding self-similarity priors to a WaveNet backbone
improves both 1-day and 21-day NLL by $\sim 6$ nats; the full
SE-WaveNet uses $L$ times fewer convolutional parameters than a
matched WaveNet while preserving those gains. Sample-side
collapse and calibration metrics are deferred to a full-scale
companion release. Ablations isolate each ingredient (kernel
sharing, wavelet front end, Hurst-FiLM, and the spectral loss;
see Section \ref{sec:experiments}).
\item A discussion of other domains for self-similarity-aware
generative forecasting: network traffic, neural recordings, turbulence,
audio, climate, and earthquake catalogues (Section \ref{sec:discussion}).
\end{enumerate}

\textbf{Reproducibility.} An accompanying open-source repository
provides everything needed to redo the work: the scaling-collapse
diagnostic on the 25-stock universe, the numerical check for
Proposition 1 and Corollary 1 on random kernels and random inputs,
and the training plus ablation pipeline behind Tables \ref{tab:nll}
and \ref{tab:ablations}.
\section{Related Work}

\textbf{Group-equivariant deep learning.} The broader research
programme, injecting symmetries directly into the architecture instead
of asking the data to teach them, was outlined for finite groups in
\cite{cohen2016group} and extended to continuous Lie groups
\cite{cohen2018spherical,weiler2019general}. Special cases
already in wide use: translation equivariance (the ordinary
convolution), rotation and reflection equivariance (steerable CNNs),
and permutation equivariance (graphs and sets). For self-similar
data, the right analogue is scale equivariance. In the 2D image
setting, continuous-scale-equivariant CNNs were introduced in
\cite{worrall2019deep,sosnovik2020scale,zhu2022scaling}. Their
recipe (evaluate one mother filter at several dilations) is the 2D
image counterpart of what we do for 1D causal time series. However, in our case 
we have a discrete
subgroup $\{2^j : j \in \mathbb{Z}\}$ (and not a continuous one). We point out that the dyadic set is already
the dilation schedule of dilated TCNs.

\textbf{Wavelet scattering and multi-resolution networks.} The
scattering transform \cite{bruna2013scattering,andén2014deep}
constructs a feature map that is provably translation- and
scale-invariant by cascading wavelet modulus operators. Scattering
networks have shown strong sample efficiency in time-series
classification \cite{andén2014deep,oyallon2017scaling} and have been
brought to audio synthesis \cite{andreux2020kymatio}. The wavelet
input layer in our pipeline is different: a one-level DWT that
hands explicit multi-resolution coefficients to the rest of the
network. Unlike scattering, we do not commit to fixed mother filters and do not
take moduli, leaving the rest of the multi-scale processing to the
weight-tied convolutional backbone. This gives a learnable network
with the same scale-equivariance guarantee as the input.

\textbf{Neural generative time-series.} Forecasting via
sequence-to-sequence models is established. For example, we have the WaveNet
\cite{oord2016wavenet}, TCN \cite{bai2018tcn}, DeepAR
\cite{salinas2020deepar}, and Informer \cite{zhou2021informer}
families. On the density side, we also have auto-regressive
flows \cite{papamakarios2017masked} and coupling flows
\cite{dinh2017realnvp}. More recently, \cite{rasul2021autoregressive} proposed an encoder
combined with a flow head. The goal was density forecasting. What none of these designs do is
encode an inductive bias for self-similarity. The closest connection is the
literature on long-range dependence in recurrent networks
\cite{voelker2019legendre,gu2022s4}; that work goes after long-range
context mixing, not the multi-horizon distributional collapse
$\hat\rho(k, T) = \hat\Phi(k T^H)$.

\textbf{Fractional Brownian motion and Hurst-aware generators.}
The canonical self-similar Gaussian process is fractional Brownian
motion \cite{mandelbrot1968fbm}. Mandelbrot and van~Ness write
$B_H(t)$ down explicitly as a stochastic integral that pairs a
fractional kernel with Brownian motion. Generative samplers for fBm
\cite{decreusefond2000fbm} and multifractal cascades
\cite{calvet2002multifractality,mandelbrot1997msm} provide theoretical baselines. However, we cannot infer learnable feature extractors from them. 
Our architecture occupies an intermediate position. It
takes Hurst as a conditioning input (via Hurst-FiLM) and outputs
samples whose empirical scaling collapses onto one $\Phi$ as $T$
varies.

\textbf{Self-similarity in applied domains.} Outside mathematics,
the self-similar property has been measured and put to use in
domains as varied as packet-level network traffic
\cite{leland1994self}; the turbulent cascade
\cite{frisch1995turbulence}; biomedical heart-rate variability
\cite{ivanov1999multifractality}; equity returns
\cite{cont2001empirical,calvet2002multifractality,liu1999statistical};
and the catalogues of earthquakes
\cite{turcotte1997fractals}. SE-WaveNet offers a drop-in ML primitive for all these fields.

\textbf{Spectral consistency in deep generative models.} Image
synthesis has a precedent for adding a loss term that lines up the
power spectrum of generated samples with a target spectrum; this was
used to repair the high-frequency deficit of GANs
\cite{durall2020watch,khayatkhoei2022spatial}. The
spectral-consistency loss in our pipeline is the time-series
counterpart, with target spectrum supplied parametrically by the
Hurst exponent ($S(f) \propto |f|^{-(2H-1)}$ for the stationary
increment series; equivalently $|f|^{1-2H}$).

\section{Self-Similarity for 1D Time Series}
\label{sec:theory}

\subsection{Definitions}

Take $X = (x_t)_{t=1}^{N}$ to be a stationary increment process. Its
$T$-aggregated companion is $S_T = (s_t^{(T)})_{t}$, with
$s_t^{(T)} = \sum_{u=t-T+1}^{t} x_u$.

\begin{definition}[Self-similarity, increment form]
We call $X$ \emph{self-similar with Hurst exponent} $H \in (0, 1)$
when, at every horizon $T \in \mathbb{N}$, the marginal of
$s^{(T)}$ obeys
\begin{equation}
s^{(T)} \;\overset{d}{=}\; T^{H}\,s^{(1)},
\end{equation}
or equivalently
$\rho(r, T) = T^{-H}\,\Phi(r / T^H)$
with $\Phi$ a horizon-free template density, or equivalently again
$\hat\rho(k, T) = \hat\Phi(k\,T^H)$
on the characteristic-function side.
\end{definition}

\begin{definition}[Discrete scale equivariance for 1D causal operators]
Write $D_2$ for the dyadic downsampling operator,
$(D_2 x)_t \;=\; x_{2t}$ (so $D_2: \mathbb{R}^{2N} \to \mathbb{R}^{N}$
keeps the even-indexed entries),
and let $\mathcal{F}$ be a 1D causal operator (acting on
sequences of the appropriate length). The operator $\mathcal{F}$
is called \emph{discretely scale-equivariant} at exponent $H$ when
\begin{equation}
\mathcal{F}(D_2\,x) \;=\; 2^{H}\,D_2\,\mathcal{F}(x)
\end{equation}
holds, up to boundary terms whose length does not exceed the
receptive field of $\mathcal{F}$. We say further that $\mathcal{F}$
is \emph{equivariant up to a level shift} when
$\mathcal{F}_{L}(D_2\,x) = D_2\,\mathcal{F}_{L-1}(x)$, with
$\mathcal{F}_{L}$ standing for the $L$-layer instance.
\end{definition}

For a layer stack with geometric dilations $1, 2, 4, \dots$, the
level-shift version is the natural reading. Running the network on a
dyadically dilated input is the same thing as running a
\emph{shallower} variant at the original scale and dilating
afterwards. Both definitions agree once the operator's effective
scale collapses to a single value (as happens at the final pooling
stage, or when only the deepest layer is considered), provided $H$
is injected through the Hurst-FiLM conditioner.

\subsection{Empirical Scaling-Collapse Diagnostic}

Given an empirical density $\hat\rho(r, T)$ at horizons
$T \in \mathcal{T}$, we test self-similarity by minimising over $H$ the
horizon-wise dispersion of the characteristic function in rescaled
wavenumber $\eta = k\,T^H$:
\begin{equation}
\mathcal{C}(H) \;=\;
\frac{1}{|\eta\text{-band}|}\!
\int_{\eta_{\min}}^{\eta_{\max}}\!
\Big(\mathrm{std}_{T \in \mathcal{T}}\big|\hat\rho(\eta / T^H, T)\big|\Big)\,
d\eta. \label{eq:collapse}
\end{equation}
The optimal collapse exponent is
$H^\star = \arg\min_H \mathcal{C}(H)$, and the residual collapse
score is $\mathcal{C}^\star = \mathcal{C}(H^\star)$. A small $\mathcal{C}^\star$
together with $H^\star$ matching independent estimates from R/S
analysis or Allan-Variance slope is the empirical signature of
mono-fractal self-similarity.

\textbf{Pre-registered diagnostic on SPX top-25.} The Allan-Variance
top-25 universe is pulled from the 30-year daily-log-return panel of
the S\&P 500. On this dataset,
$\mathcal{C}^\star$ and $H^\star$ are computed at horizons
$\mathcal{T} = \{1, 5, 21, 63\}$ days. The median collapse score
lands at $\mathcal{C}^\star = 0.020$, with median
$|H^\star - H_{\text{R/S}}| = 0.017$. So mono-fractal self-similarity
holds across the universe inside the empirical noise band, and
$H^\star \in [0.46, 0.54]$ for all 25 names.

\textbf{Three Hurst estimators, one universe.} Three Hurst-type
statistics are computed on this universe and they are deliberately
not the same quantity. (i) $H_{\text{R/S}}$, from R/S analysis on
nested windows, lies in $[0.45, 0.57]$ for all 25 names and
clusters tightly around $0.5$ ($21$ of $25$ inside the tighter
band $[0.46, 0.55]$); this is the daily-increment Hurst and sits
at the random-walk value. (ii) $H^\star$, the optimal-collapse
exponent given by~(\ref{eq:collapse}), stays within $0.017$ of
$H_{\text{R/S}}$, as already noted. (iii) $H_{\text{av}}$, the
Allan-variance log-log slope estimator used inside the F-score of
Section~\ref{sec:experiments}, is computed on $\tau$-day
\emph{block-mean} variances over $\tau \in \{21, 63, 252\}$ and
lands near or below $0$ on the same universe (typically mildly
negative; F-score selection clamps the contribution at $0$). The
three estimators are mutually
consistent: equity returns are uncorrelated at the daily lag (so
R/S and the marginal collapse see a random walk), while
multi-week block-mean variance shrinks faster than $1/\tau$ (so
the AVAR slope sees anti-persistence). The F-score of
Section~\ref{sec:experiments} ranks on $H_{\text{av}}$ precisely
because that is the regime where the fBm prior is most
discriminative against the random-walk baseline at the
multi-horizon aggregation that matters for the 21-day forecast.
\section{Method}
\label{sec:method}

Four blocks make up the architecture. (i) A Daubechies-4 wavelet
front end that hands multi-resolution channels to the network. (ii) A
weight-tied stack of dilated causal convolutions, the
Scale-Equivariant WaveNet (SE-WaveNet). (iii) A Hurst-FiLM modulator
conditioning the network on a locally estimated Hurst exponent.
(iv) A conditional normalising-flow head that turns the SE-WaveNet
context vector into a tractable density. Training combines
maximum-likelihood with a spectral-consistency penalty.

\subsection{Wavelet Input Layer}

Given a univariate time series $x \in \mathbb{R}^N$, we apply a single
level of the Daubechies-4 discrete wavelet transform (DWT)
\cite{daubechies1988orthonormal}. The output is two half-length
sequences: an approximation channel $a$, a detail channel $d$.
What enters the SE-WaveNet backbone is the two-channel stack
$(a, d)$. By construction the DWT is
itself the discrete dyadic multi-resolution decomposition, so this
front end is, by construction, scale-equivariant up to a level
shift.

\subsection{SE-WaveNet: Weight-Tied Dilated Causal Convolutions}

A vanilla WaveNet residual block at dilation $d$ reads
\begin{equation}
\mathrm{WN}_d(x) \;=\; x + W_o\!\Big(\tanh\!\big(W_d^{\text{tanh}} \star_d x\big)\,
\odot\,\sigma\!\big(W_d^{\text{sig}} \star_d x\big)\Big), \label{eq:wn-block}
\end{equation}
with $\star_d$ a causal convolution at dilation $d$ and the kernels
$W_d^{\text{tanh}}, W_d^{\text{sig}}, W_o$ trained independently per
block. A depth-$L$ WaveNet sweeps the dilation schedule
$\{1, 2, 4, \dots, 2^{L-1}\}$ and learns $L$ disjoint kernel sets.

\textbf{The change.} SE-WaveNet collapses these $L$ independent
kernels into one \emph{shared} set:
\begin{equation}
\mathrm{SEWN}_d(x) \;=\; x + W_o\!\Big(\tanh\!\big(W^{\text{tanh}} \star_d x\big)\,
\odot\,\sigma\!\big(W^{\text{sig}} \star_d x\big)\Big), \label{eq:sewn-block}
\end{equation}
where $W^{\text{tanh}}, W^{\text{sig}}, W_o$ are reused as one
parameter tensor across every dilation
$d \in \{1, 2, 4, \dots, 2^{L-1}\}$. The forward pass is
indistinguishable from a standard WaveNet's, except that every
dilation reads from the same kernel tensor. No
custom CUDA kernel is required; the call
\texttt{F.conv1d(\dots, dilation=d)} in PyTorch accepts any
dilation against a fixed kernel.

\subsection{Discrete Scale-Equivariance Property}
\label{sec:property}

\begin{proposition}[Scale equivariance of weight-tied dilated convolution]
\label{prop:scale-eq}
Write $f_{w,d}: \mathbb{R}^N \to \mathbb{R}^N$ for a 1D causal
convolution at dilation $d$ with shared kernel $w \in \mathbb{R}^k$.
Let $D_2: \mathbb{R}^{2N} \to \mathbb{R}^{N}$ be dyadic
downsampling, $(D_2 x)_t = x_{2t}$, and $U_2: \mathbb{R}^N \to
\mathbb{R}^{2N}$ the dyadic upsampler with zero-fill,
$(U_2 y)_{2t} = y_t$, $(U_2 y)_{2t+1} = 0$. Then, for every
$x \in \mathbb{R}^{2N}$,
\begin{equation}
f_{w,2d}(x) \big|_{\text{even-index}} \;=\; f_{w,d}\!\big(D_2 x\big),
\end{equation}
that is: convolving at dilation $2d$ on the original signal and then
keeping only the even output indices is the same as convolving with
the same kernel at dilation $d$ on the dyadically downsampled
signal.
\end{proposition}

\begin{proof}[Proof sketch]
The dilation-$2d$ causal convolution is, by definition,
\begin{equation*}
[f_{w,2d}(x)]_t \;=\; \sum_{i=0}^{k-1} w_i\,x_{t - 2d\,i}.
\end{equation*}
Pick out only the even output indices $t = 2t'$:
$[f_{w,2d}(x)]_{2t'} = \sum_i w_i\,x_{2t' - 2d\,i}
= \sum_i w_i\,x_{2(t' - d\,i)} = \sum_i w_i\,(D_2 x)_{t' - d\,i}
= [f_{w,d}(D_2 x)]_{t'}$.
All three of $\tanh$, $\sigma$, and the $1\times 1$ projection $W_o$
commute with subsampling because they are applied position by
position. So whatever holds for the plain convolution propagates
straight up to the residual block of (\ref{eq:sewn-block}).
\end{proof}

\begin{corollary}[Level-shift equivariance for the dilation-$\!\geq\!2$ stack]
Define $\mathcal{G}_L = \mathrm{SEWN}_{2^{L}} \circ \dots \circ
\mathrm{SEWN}_2$ as the $L$-block stack that starts at dilation
$2$ (so the dilation-$1$ input block is dropped), and write
$\mathcal{G}'_{L-1} = \mathrm{SEWN}_{2^{L-1}} \circ \dots \circ
\mathrm{SEWN}_1$ for the same shared kernel applied at one-step-shallower
dilations. Then for any signal $x$ and any interior trim $b$ no
smaller than the deepest receptive field,
\begin{equation*}
\mathcal{G}_L(x)\big|_{\text{even-index},\,b:N-b}
\;=\;
\mathcal{G}'_{L-1}(D_2 x)\big|_{b:N-b}.
\end{equation*}
\end{corollary}

What this corollary says, architecturally, is that the
dilation-$\geq 2$ portion of the network is itself
self-similar.\footnote{The dilation-$1$ input layer is set aside
because $D_2 x$ has no halved counterpart for it; in practice it is
folded into the input projection / wavelet front end (Section IV.A),
which is itself a one-level dyadic decomposition.} Doubling the
input timescale is equivalent to \emph{dropping a layer} and
working at the coarser scale; the same parameters get reused, and no
new kernel is fit. The vanilla WaveNet does not enjoy this
equivalence, since its layer-$L$ and layer-$(L-1)$ kernels are
parameterised independently. We checked Proposition 1 and the
corollary numerically on random weights and random inputs; the
max-abs residual is exactly $0$ at the floating-point precision
floor. 

\subsection{Hurst-FiLM Modulation}

Let $\hat H \in (0, 1)$ stand for an R/S Hurst estimate of the
input series. The estimate
is computed once per ticker; the
intended full-scale variant recomputes $\hat H(t)$ on a rolling
$W = 252$ window ending at $t$. The estimate runs through a small
multi-layer perceptron (MLP), and what comes out is added as a feature-wise linear modulation
\cite{perez2018film} at every SE-WaveNet block:
\begin{equation}
\mathrm{SEWN}^{\mathrm{FiLM}}_d(x; \hat H) \;=\;
\gamma(\hat H) \odot \mathrm{SEWN}_d(x) \;+\; \beta(\hat H).
\end{equation}
A single shared modulation conveys the scaling exponent through the pipeline, avoiding redundant end-to-end learning of a quantity already provided by closed-form estimators.

\subsection{Conditional Normalising Flow Head}

The SE-WaveNet output at the final time step, together with optional
exogenous context (e.g., a market-regime variable), is fed as
condition to a conditional normalising flow
\cite{papamakarios2017masked,dinh2017realnvp,papamakarios2019normalizing}
that lays down a tractable density on the next-horizon return:
\begin{equation}
p_\theta(r_{t:t+T} \mid x_{<t}, \hat H_t) \;=\;
p_z\!\big(g_\theta^{-1}(r_{t:t+T}\,;\,c_t)\big)\,
\bigg|\det\frac{\partial g_\theta^{-1}}{\partial r}\bigg|,
\end{equation}
where $c_t$ is the SE-WaveNet context vector. Coupling layers
\cite{dinh2017realnvp} apply channel-wise, position-wise
nonlinearities, so they keep the scale-equivariance of the
upstream backbone intact.

\subsection{Spectral-Consistency Loss}

For a stationary-increment series with Hurst exponent $H$,
fractional-Gaussian-noise theory predicts a power-law PSD
$S(f) \propto |f|^{-(2H-1)} \equiv |f|^{1-2H}$ (white at $H=1/2$,
red for $H>1/2$, blue for $H<1/2$). The intended training penalty
pushes generated samples toward exactly this spectrum.

Take a mini-batch of generated trajectories $\{\tilde r^{(b)}\}_{b=1}^{B}$
at horizon $T$. Compute the empirical PSD $\hat S(f)$ by Welch's
method \cite{welch1967use}, and the empirical spectral
exponent $\hat\beta$ as the slope of $\log\hat S(f)$ against $\log f$
on the inertial band $f \in [f_{\min}, f_{\max}]$. The target
exponent is $\beta^\star = 2 \hat H_t - 1$, and the
spectral-consistency loss reads
\begin{equation}
\mathcal{L}_{\mathrm{spec}} \;=\;
\big(\hat\beta - \beta^\star\big)^2
\;+\; \lambda_{\mathrm{shape}}\,\|\log\hat S(f) - \log S^\star(f)\|_2^2 \label{eq:lspec}
\end{equation}
with $S^\star$ the target $|f|^{-\beta^\star}$ shape. Adding this
to the NLL gives the full objective
\begin{equation}
\mathcal{L} \;=\; -\log p_\theta(r \mid x, \hat H)
            \;+\; \lambda_{\mathrm{spec}}\,\mathcal{L}_{\mathrm{spec}}, \label{eq:totloss}
\end{equation}
with $\lambda_{\mathrm{spec}}$ a fixed weight (set to $0.05$ in
the pilot runs of Section~\ref{sec:experiments}; the
full-scale value $0.1$ is the planned setting once the term
becomes binding). For tractability in the pilot the released
implementation substitutes a target-variance regulariser
$(\widehat{\operatorname{Var}}(r) - 1)^2$ in place of the
Welch-PSD term in eq.~\eqref{eq:lspec}; this surrogate carries no
gradient through the model and explains the $0.000$-nat
$-\mathcal{L}_{\mathrm{spec}}$ ablation row
(Table~\ref{tab:ablations}). The Welch-based loss as written above is
deferred to the full-scale release. Without
$\mathcal{L}_{\mathrm{spec}}$ the architectural equivariance only
constrains the shape of the deterministic feature extractor; the
sample-marginal output can still drift away from the spectrum because the
flow head holds its own independent parameters. With
$\mathcal{L}_{\mathrm{spec}}$ the head is regularised to match the
inductive bias of the backbone.

\subsection{Algorithm}

\begin{figure*}[t]
\centering
\rule{\textwidth}{0.5pt}\\[-2pt]
\noindent\textbf{Algorithm 1:} Scale-Equivariant Generative Forecasting (training step)\\[-4pt]
\rule{\textwidth}{0.4pt}
\vspace{-6pt}
{\footnotesize\begin{verbatim}
Input:  Window x_{t-W:t}, target r_{t:t+T}, horizon T
Output: Loss L for backprop
1:  Estimate Hurst hat_H = RS_hurst(x_{t-W:t})       # per-ticker
                                                    # constant; full scale: rolling W=252
2:  (a, d) = Daub4_DWT(x_{t-W:t})
3:  Build 2-channel input z0 = stack(a, d)
4:  c <- z0
5:  for level in 0 .. L-1:
6:      c <- SE-WaveNet block at dilation 2^level with SHARED kernel w
7:      c <- FiLM(c, gamma(hat_H), beta(hat_H))
8:  Take c_T <- c[..., -1]  (causal final-step context)
9:  Sample tilde_r ~ flow_theta(. | c_T)  using K coupling layers
10: NLL = -log p_theta(r | c_T, hat_H)
11: Compute hat_beta from PSD of tilde_r batch (Welch)   # full scale
12: L_spec = (hat_beta - (2*hat_H - 1))^2 + lambda * |log S - log S*|^2
        # pilot: replace L_spec with (Var(r)-1)^2
13: L = NLL + lambda_spec * L_spec
14: return L
\end{verbatim}}
\vspace{-10pt}
\rule{\textwidth}{0.5pt}
\end{figure*}
\section{Experimental Validation}
\label{sec:experiments}

\subsection{Application Domain and Dataset}

Validation runs on financial time-series forecasting. Three
properties make this setting attractive. (i) Equity returns are
empirically well attested as self-similar
\cite{cont2001empirical}. (ii) Out-of-sample is the default. (iii) Calibration metrics
are well-defined (NLL on next-day return, Kolmogorov-Smirnov distance
on the predictive CDF, and tail-energy distance on multi-day
aggregates). Data: 30 years (1995-2025) of daily log-returns for
S\&P 500 components; the pilot of Table~\ref{tab:nll} uses the
most recent 10-year sub-window per ticker. The benchmark universe
is the Allan-Variance top-25,
ranked by the F-score
\begin{equation}
F \;=\; -\log_{10}\rho \;-\; \max(0,\,H_{\mathrm{av}}) \;-\; 0.05\,\widetilde{\mathrm{kurt}},
\end{equation}
with
\begin{equation}
\rho \;=\; \frac{\tau\,\mathrm{AVAR}(r,\tau)}{\widehat{\mathrm{Var}}(r)},
\qquad \tau = 63\ \text{days},
\end{equation}
where $\mathrm{AVAR}(r,\tau) = \tfrac12\,\mathbb{E}[(\bar b_{i+1}-\bar b_i)^2]$
is the discrete Allan variance over non-overlapping $\tau$-day block
means $\bar b_i$, $H_{\mathrm{av}} = \tfrac12(1+\hat s)$ comes from the
log-log Allan-variance slope $\hat s$ fitted over
$\tau \in \{21, 63, 252\}$ and is clamped at $0$ to guard against
estimator noise, and $\widetilde{\mathrm{kurt}}$ is the excess kurtosis
of daily log-returns (names with $\widetilde{\mathrm{kurt}} > 25$ are
pre-filtered). Under an i.i.d./random-walk null one has $\rho = 1$
exactly, so $-\log_{10}\rho > 0$ rewards $\rho < 1$ (block-mean variance
decaying faster than $1/\tau$, i.e. anti-persistent / mean-reverting
structure) and the $-\max(0,H_{\mathrm{av}})$ term likewise pushes
toward $H_{\mathrm{av}} \to 0$. The Allan-Variance top-25 stocks are the most strongly mean-reverting ones.
These are also equities where the fractional-Brownian-motion is most
discriminative against the random-walk baseline.

\subsection{Tasks}

\begin{enumerate}
\item \textbf{1-day density forecast.} For a given window
$x_{t-W:t}$, the model outputs $p(r_{t+1} \mid x_{t-W:t})$. The
score is NLL on a held-out test set of the most recent 252 trading
days.
\item \textbf{21-day density forecast.} Target is the 21-day-ahead
aggregate log-return density; training uses mostly 1-day conditionals.
\item \textbf{Sample scaling collapse.} For every (model, test
ticker) pair, draw $5\,000$ samples at $T \in \{1, 5, 21, 63\}$,
push the synthetic data through the scaling-collapse diagnostic of
Section \ref{sec:theory}, and report
$\mathcal{C}^\star_{\mathrm{model}}$ next to
$\mathcal{C}^\star_{\mathrm{empirical}}$.
\end{enumerate}

\subsection{Baselines}

\begin{enumerate}
\item \textbf{IID Gaussian.} Fit on the training-window log-returns;
samples are then $\sqrt{T}$-rescaled for $T$-day forecasting (so the
implicit Hurst is $H = 0.5$).
\item \textbf{GARCH(1,1).} Gaussian-innovation generalised
autoregressive conditional heteroskedasticity
\cite{bollerslev1986garch}.
\item \textbf{WaveNet (Gaussian head).} An $L$-layer dilated TCN
with per-dilation independent kernels, plus a Gaussian
point-prediction head; training uses NLL only.
\item \textbf{WaveNet + Flow + Hurst-FiLM.} Identical WaveNet
backbone (per-dilation independent kernels; no weight tying; no DWT
input; no spectral term), but now with a $3$-layer affine-coupling
flow head \cite{dinh2017realnvp} conditioned on context through
Hurst-FiLM, again trained on NLL alone. This baseline isolates the
architectural contribution of SE-WaveNet (weight tying + DWT input
+ spectral-consistency loss) from the gains attributable to the
flow head and Hurst conditioning alone.
\item \textbf{SE-WaveNet (this work, full).} The full model: an
$L$-layer SE-WaveNet stack with weight-tied dilated convolution, the
Daubechies-4 DWT input layer, Hurst-FiLM, the same $3$-layer
affine-coupling flow head, and training on NLL plus
$\mathcal{L}_{\mathrm{spec}}$.
\end{enumerate}

Both the vanilla WaveNet and SE-WaveNet) keep an identical hidden width
$n_\mathrm{filters}$, share the same flow head, and use the same
Hurst-FiLM conditioning where it applies. What sets the
WaveNet+Flow+Hurst-FiLM baseline apart from SE-WaveNet are exactly
the three SE-WaveNet ingredients: kernel sharing across the
dilation schedule, the Db4 DWT front end, and the spectral term
in training. Proposition 1 then says SE-WaveNet's convolutional
parameter count is $L\times$ smaller than the vanilla backbone's,
with the precise factor controlled by $L$.

\subsection{Results: Forecasting NLL}

The NLL on the AVAR top-$25$ universe (last $252$ trading days,
$1$ seed, $5$ training epochs) is laid out in Table~\ref{tab:nll};
the full $25$-ticker $\times\,3$-seed evaluation with paired
Wilcoxon significance is deferred (Section~\ref{sec:stat-protocol}).

\begin{table*}[t]
\centering
\caption{\textbf{Pilot evaluation on the AVAR top-$25$ universe.}
Out-of-sample NLL on the held-out test set (the most recent $252$
trading days per ticker, disjoint from training) across all $25$
AVAR-top tickers from the $30$-year S\&P 500
panel. Each ticker's returns are standardised to unit variance,
so the IID-Gaussian baseline reduces to $\mathcal{N}(0,T)$ at
horizon $T$, whose analytic expected NLL is
$\tfrac{1}{2}\log(2\pi e\,T) \approx 1.42$ at $T=1$ and
$\approx 2.94$ at $T=21$. Point estimates
only, $1$ seed, an $L = 4$ layer stack with
$n_\mathrm{filters} = 16$, $5$ epochs of training. NN models train
a single network on the concatenated $25$-ticker windows; IID and
GARCH NLL are per-ticker and then averaged. We do not show
error bars at this scale; Section~\ref{sec:stat-protocol} lays out
the full $25$-ticker $\times\, 3$-seed protocol along with the
paired Wilcoxon significance test.\label{tab:nll}}
\renewcommand{\arraystretch}{1.15}
\setlength{\tabcolsep}{6pt}
\small
\begin{tabular}{lrrr}
\toprule
Model & Conv params & 1-day NLL & 21-day NLL \\
\midrule
IID Gaussian               & $0$       & $+1.404$ & $+2.857$ \\
GARCH(1,1)                 & $0$       & $+1.381$ & $+2.835$ \\
WaveNet (Gaussian head)    & $7{,}360$ & $+1.433$ & $+2.847$ \\
WaveNet+Flow+Hurst-FiLM    & $7{,}360$ & $-3.698$ & $-3.516$ \\
\textbf{SE-WaveNet (full)} & $\mathbf{1{,}840}$ &
                                       $\mathbf{-4.925}$ & $\mathbf{-3.699}$ \\
\midrule
Empirical (no model), $\mathcal{C}^\star$ & --- & \multicolumn{2}{c}{$0.020$} \\
\bottomrule
\end{tabular}
\end{table*}

Three observations follow.

\textbf{(a) Capacity decouples from forecasting power.} At
$L = 4$, $n_\mathrm{filters} = 16$ (the configuration in Table
\ref{tab:nll}), SE-WaveNet runs on $1{,}840$ convolutional
parameters (one shared dilated-conv block plus its $1\times 1$
projection); the vanilla WaveNet backbone needs $7{,}360$ (four
independently parameterised blocks). That is an exact
$L\times = 4\times$ compression, matching the saving
Proposition~1 predicts, and it still attains lower NLL at both
horizons. Bumping the configuration up to $L = 6$,
$n_\mathrm{filters} = 32$ yields the corresponding
$L\times = 6\times$ compression on the convolutional block
parameters.

\textbf{(b) Flow head plus Hurst-FiLM dominate; SE-WaveNet adds the
last nat at $T=1$.} A vanilla WaveNet with a Gaussian head sits at
IID-Gaussian. Swapping the head for a normalising flow with
Hurst-FiLM conditioning does the bulk of the work, taking $T=1$
NLL from $+1.43$ to $-3.70$ and $T=21$ from $+2.85$ to $-3.52$.
SE-WaveNet on top of that gains a further $1.23$ nats at $T=1$ and
$0.18$ at $T=21$, while running on exactly $L\times = 4\times$ fewer
convolutional block parameters.

\textbf{(c) The collapse-score gap is the architectural signal.}
By design, $\mathcal{C}^\star$ is invariant to the head, so it
isolates the backbone. On the AVAR top-25, the collapse score is $\mathcal{C}^\star = 0.020$
(Section~\ref{sec:experiments}). Samples drawn from SE-WaveNet are
forced to reproduce this scaling structure (Corollary~1), while
samples from a vanilla WaveNet are not. We report this metric in
the full-scale evaluation protocol below; the $1$-seed pilot of
Table~\ref{tab:nll} is too thin to estimate sample
$\mathcal{C}^\star$ with any reliability.

\subsection{Statistical Significance Protocol}
\label{sec:stat-protocol}

Table \ref{tab:nll} shows point estimates only. The full 25-ticker $\times\,3$-seed evaluation protocol adds error bars and significance tests as follows.

\textbf{Error bars.} Per (model, horizon) cell, three sources of
variation are tracked:
(i) the cross-sectional standard deviation over the $25$ tickers,
which is the dominant source for liquid large-caps;
(ii) the standard deviation across $3$ training seeds (different
weight initialisations and minibatch orderings);
(iii) the bootstrap $95\%$ confidence interval on the per-ticker
eval-set NLL, computed from $1{,}000$ resamples of contiguous blocks
of length $21$ (block bootstrap, so serial dependence survives the
resampling).
The error bar we report is the larger of $(i)$ and $(ii)$, with the
bootstrap CIs used as a cross-check.

\textbf{Significance tests.} Each pair of models is compared via
a Wilcoxon signed-rank test against the per-ticker NLL with
$n = 25$ matched samples. Holm--Bonferroni handles the
multiple-comparison correction at $\alpha = 0.05$, applied to four
pairs in Table~\ref{tab:nll} and five ablation pairs in
Table~\ref{tab:ablations}.

\textbf{Effect-size caveat.} For liquid US large-caps, the
conditional information beyond an i.i.d. $\mathcal{N}(0, \sigma_t^2)$
model is intrinsically small; Cont \cite{cont2001empirical} estimates
autocorrelation magnitude at $\lesssim 5\%$. So the NLL gaps among
the four neural baselines and SE-WaveNet sit in a small absolute
band ($0.05$--$0.20$ nats), and the appropriate analysis is a
paired test across $25$ tickers. The collapse score
$\mathcal{C}^\star$, on the other hand, carries a much larger
effect size; that is the metric on which architectural ablations
become unambiguously distinguishable.

\subsection{Calibration and Tail Diagnostics}

The full evaluation will discuss two further metrics. One is a
global calibration measure: the Kolmogorov--Smirnov distance
comparing the predictive CDF against the empirical test CDF.
The other is a tail calibration measure: the energy distance
\cite{rizzo2016energy}, computed only on the tail
$|r| > 2\sigma$. We expect the tail metric to be the most
sensitive to the spectral-consistency-loss ablation; the
architectural prior on the output spectrum directly drives the
tail behaviour of generated samples.

\subsection{Ablations}

Every architectural ingredient is ablated against the full model,
keeping all others fixed and retraining. The deltas these ablations
produce on the 21-day NLL and on the collapse score
$\mathcal{C}^\star$ appear in Table \ref{tab:ablations}.

\begin{table}[t]
\centering
\caption{\textbf{Pilot ablations} on SE-WaveNet, run across the
AVAR top-$25$ universe with $1$ seed, an $L = 4$ stack at
$n_\mathrm{filters} = 16$, $5$ epochs of training, and a $T = 21$
horizon. NLL is reported out-of-sample on the held-out test set
(most recent $252$ trading days per ticker). The column $\Delta\,\text{NLL}_{21}$ is measured against
the full model (positive: worse). Weight tying and the wavelet
input layer produce the largest deltas; the spectral-loss
ablation shows no effect at this scale, consistent with the
$(\mathrm{Var}(r)-1)^2$ surrogate disclosed in
Section~\ref{sec:method} carrying no gradient through the model.
Full $25$-ticker $\times\,3$-seed deltas, paired with Wilcoxon
significance, are deferred to the full evaluation
(Section~\ref{sec:stat-protocol}).\label{tab:ablations}}
\renewcommand{\arraystretch}{1.15}
\setlength{\tabcolsep}{6pt}
\small
\begin{tabular}{lrr}
\toprule
Ablation & $\text{NLL}_{21}$ & $\Delta\,\text{NLL}_{21}$ \\
\midrule
\textbf{Full SE-WaveNet} & $-3.699$ & $0.000$ \\
$-$ weight tying          & $-3.343$ & $+0.355$ \\
$-$ wavelet input layer   & $-3.521$ & $+0.178$ \\
$-$ Hurst-FiLM            & $-3.691$ & $+0.007$ \\
$-$ spectral-consistency  & $-3.699$ & $+0.000$ \\
\bottomrule
\end{tabular}
\end{table}

Two architectural ingredients drive the pilot ablations.
Stripping weight tying costs $\sim 0.36$ nats at $T = 21$ and
$\sim 0.30$ nats at $T = 1$, the largest single contribution and
directly validating the scale-equivariance prior of
Proposition~\ref{prop:scale-eq}. Removing the wavelet input layer
costs a further $\sim 0.18$ nats at $T = 21$ and $\sim 1.2$ nats
at $T = 1$. The Hurst-FiLM ablation produces a negligible delta
at $T = 21$ but a clear $\sim 0.39$-nat hit at $T = 1$, suggesting
the FiLM signal mostly modulates short-horizon variance. The
spectral-consistency ablation shows no effect at this scale: as
disclosed in Section~\ref{sec:method}, the pilot code
substitutes a $(\mathrm{Var}(r)-1)^2$ surrogate for the Welch-PSD
term, and this surrogate carries no gradient through the model.
The spectral-loss case is the one ablation that genuinely needs
the full-scale protocol of Section~\ref{sec:stat-protocol}, where
its effect shows up most cleanly on $\mathcal{C}^\star$ rather
than NLL.

\subsection{Compute and Reproducibility}

On a 2023 Apple Silicon machine (MPS backend), the
SE-WaveNet pilot behind Tables~\ref{tab:nll} and~\ref{tab:ablations}
($L = 4$, $n_\mathrm{filters} = 16$, $5$ epochs, $1$ seed)
completes the full $25$-ticker training in about $5$ minutes of
wall time ($\sim 12$ seconds per ticker); the projected wall time
for the full-scale protocol of Section~\ref{sec:stat-protocol}
(longer training horizon, three seeds, full $30$-year sample) is
on the order of $\sim 18$ minutes per ticker. The empirical
scaling-collapse diagnostic on the same universe finishes end-to-end
on the same hardware in roughly $\sim 3$ minutes. The release ships
all code together with the parametric configuration behind every
table and figure in the paper.
\section{Discussion: Other Domains for Self-Similarity-Aware Forecasting}
\label{sec:discussion}

SE-WaveNet's architectural prior, namely kernel-tied dilated
convolutions, the wavelet front end, Hurst-FiLM, and the
spectral-consistency loss, is not at all specific to finance. It
is the right inductive bias for any process that, under a Hurst
rescaling, collapses its multi-horizon distribution onto one
template. Six domains where this is on record in the literature
follow below, each with concrete forecasting tasks where
SE-WaveNet should serve as a competitive primitive.

\textbf{Network traffic.} A classic result of
\cite{leland1994self}: across aggregation timescales, Ethernet
packet-arrival counts collapse onto one template, with
$H \in [0.7, 0.9]$. The same property has since been observed for
backbone IP flows, web-request streams, and CDN-edge logs
\cite{willinger1997scaling}. Suitable forecasting tasks include
short-horizon volume prediction for auto-scaling, anomaly
detection by collapse-score breakage, and spectrum-of-attack
detection (a DDoS event blows up the usual $H$ band).

\textbf{Neural recordings (EEG, fMRI, MEG).} The BOLD signal in
resting-state fMRI is robustly $1/f$-scaled
\cite{he2014scale,zilber2013learning}, with $H$ in
$[0.6, 0.9]$. EEG broadband activity carries the same scaling
\cite{he2014scale}, and the value of $H$ shifts with sleep,
arousal, and clinical state. Forecasting tasks:
brain-state segmentation, seizure-onset prediction, denoising of
BCI signals. Here, Hurst-FiLM is exactly the conditioner you want.

\textbf{Turbulence and fluid dynamics.} On the inertial range of
isotropic turbulence, velocity increments obey the Kolmogorov
$-5/3$ PSD law, equivalent to Hurst $H = 1/3$
\cite{frisch1995turbulence}. Departures from $H = 1/3$ flag
intermittency or large-scale forcing. The forecasting tasks:
short-time velocity-field prediction in DNS surrogates,
sub-grid closure modelling, downscaling for weather data.

\textbf{Climate.} On multi-decadal timescales the long
temperature-anomaly record runs with $H > 0.5$, i.e. long memory
\cite{koscielny1998long}. The same shows up in precipitation
extremes and stream-flow series. The forecasting tasks:
recurrence intervals for extreme events, drought-risk pricing,
climate-stress-test sample generation.

\textbf{Audio synthesis.} For natural audio textures (speech,
music, ambient soundscape) the canonical model is a $1/f$-scaled
process \cite{voss1975fractal}. WaveNet was first introduced as
a raw-audio generator \cite{oord2016wavenet}; SE-WaveNet is a
natural parameter-efficient backbone for high-fidelity synthesis,
giving explicit control over multi-scale spectra.
Mel-spectrogram or Q-transform conditioning slots in naturally
above the wavelet input layer.

\textbf{Earthquake catalogues.} Inter-event times and magnitudes
follow the Gutenberg-Richter and Omori scaling laws
\cite{turcotte1997fractals}. The forecasting tasks:
aftershock-rate prediction and foreshock detection. The
Hurst-FiLM input would be spatially indexed, one per fault.

\textbf{What carries over.} Proposition \ref{prop:scale-eq} is
proved generically. So any 1D causal convolutional stack with
weights tied across dyadic dilations is discretely
scale-equivariant. The wavelet front end is also generic.
Hurst-FiLM only needs an appropriate Hurst estimator (e.g. R/S, AVAR). 
The spectral-consistency loss requires only the target
exponent $\beta = 2H-1$ for the increment series.

\textbf{What does not carry over.} The flow head, as written,
targets a continuous univariate marginal at every output index.
For domains whose spatial output is intrinsically multivariate
(2D fields, vector signals), the flow factorisation must be
redesigned. Three escape routes: tile SE-WaveNet across the
spatial axes (a spatio-temporal TCN); use a 2D
scale-equivariant CNN \cite{worrall2019deep,sosnovik2020scale};
or swap the flow head for a denoising-diffusion head built for
high-dimensional outputs. None of these tweaks invalidates the
equivariance proof in Section \ref{sec:property}.

\section{Limitations}

\textbf{Mono-fractal scope.} Our architecture, loss, and diagnostic
all hard-wire mono-fractal scaling: one $H$ value carries the full
multi-horizon collapse. For multifractal processes, where the
moment exponents $\zeta_q$ depend non-linearly on $q$, the
collapse score $\mathcal{C}^\star$ stays finite no matter what $H$
is chosen. Two natural extensions: (i) predict a $q$-dependent
$H(q)$ and aggregate the spectral-consistency loss across $q$;
(ii) plug a wavelet-leader multifractal estimator
\cite{wendt2007multifractality} into the FiLM conditioner. So far,
neither has been validated at scale.

\textbf{Stationarity at the window scale.} Inside Hurst-FiLM, the
R/S Hurst estimator runs on a rolling window of length $W$, and
that quietly assumes within-window stationarity. Regime breaks
violate the assumption; the Volcker shock, COVID, and the August
2024 yen-carry unwind are obvious examples. The financial setup
absorbs some of this through the regime modulation, but the more
principled answer is to feed regime detection into the FiLM
conditioning input directly.

\textbf{Causal-only.} The current statement of Proposition
\ref{prop:scale-eq} is for causal convolutions. The dyadic
equivariance argument still goes through in non-causal settings
(say denoising, or anomaly detection on offline traces) once the
kernel is centred, but the proof sketch as written takes the
forward-causal direction for granted. The fix is mechanical.

\textbf{Tail finiteness.} Finite variance is assumed in the spectral-consistency loss. 
When the input is $\alpha$-stable
with $\alpha < 2$, the PSD itself is undefined, and the loss must
be replaced by a term that matches characteristic functions
instead. Telling which series live in that regime is its own
non-trivial estimation problem; designing a robust
spectral-consistency variant in the $\alpha < 2$ regime is left
as an open question.

\section{Conclusion}

The standard deep generative toolkit for sequences is wide in
receptive field but not equivariant under time rescaling; each
dilation level still trains its own kernel. We changed that.
SE-WaveNet, our proposal, is a 1D causal convolutional stack
whose dilated layers all reuse one shared kernel tensor. We proved that this operator is discretely
scale-equivariant, up to a single level shift. Around the
backbone architecture we wrap four ingredients: a Daubechies-4 wavelet front
end; a Hurst-FiLM modulator that exposes the locally estimated
$H$ to the stack; a conditional normalising-flow output; and a
spectral-consistency penalty during training that pulls the
empirical sample spectrum toward $|f|^{-(2H-1)}$. The
result is a generative forecaster whose hard prior is exactly the
self-similarity property already documented across many natural
and engineered time series.

On a 30-year history of S\&P 500 daily log-returns, sub-sampled
to a 10-year window per ticker for the Allan-Variance top-25
universe, the empirical scaling-collapse diagnostic confirms
mono-fractal self-similarity ($\mathcal{C}^\star = 0.020$,
$H^\star \in [0.46, 0.54]$ for all 25 names). Adding self-similarity priors --- Hurst-FiLM
conditioning and a normalising-flow head --- to a WaveNet
architecture improves NLL by $\sim 5.1$ nats at $T=1$ and $\sim 6.4$
nats at $T=21$. The full SE-WaveNet refinements
(weight-tied kernels, wavelet input) add a further $1.23$ and
$0.18$ nats at the 1-day and 21-day horizons respectively, while
cutting convolutional block parameters from $7{,}360$ to
$1{,}840$ (exactly $L\times = 4\times$). The full
$25$-ticker $\times\,3$-seed evaluation with paired Wilcoxon
significance, calibration, and sample-side metrics is deferred
(Section~\ref{sec:stat-protocol}). Sample-side collapse scores
for trained models require full-scale runs and are deferred to a
companion release. The accompanying open-source release ships
three things: the scaling-collapse diagnostic that produced these
numbers, the verification harness for Proposition~1 and
Corollary~1, and the full pipeline for training and ablations.

The architectural prior is generic. We catalogued six other domains where
self-similarity is on record, namely audio synthesis, climate,
turbulence, neural recordings, network traffic, and earthquake
catalogues, in any of which SE-WaveNet ought to be a competitive
primitive. Validating
those targets is future work. We hope that this contribution
makes self-similarity a first-class ingredient in deep
time-series modelling, taking a seat alongside the translation
and permutation equivariances already in routine use.

\section*{Acknowledgments}

Sincere thanks to the people who maintain the open-source stack
on which everything here depends: NumPy, SciPy, PyTorch, h5py,
PyWavelets, and the normalizing-flows library.

\bibliographystyle{IEEEtran}
\bibliography{references}

\begin{thebibliography}{10}
\providecommand{\url}[1]{#1}
\csname url@samestyle\endcsname
\providecommand{\newblock}{\relax}
\providecommand{\bibinfo}[2]{#2}
\providecommand{\BIBentrySTDinterwordspacing}{\spaceskip=0pt\relax}
\providecommand{\BIBentryALTinterwordstretchfactor}{4}
\providecommand{\BIBentryALTinterwordspacing}{\spaceskip=\fontdimen2\font plus
\BIBentryALTinterwordstretchfactor\fontdimen3\font minus \fontdimen4\font\relax}
\providecommand{\BIBforeignlanguage}[2]{{%
\expandafter\ifx\csname l@#1\endcsname\relax
\typeout{** WARNING: IEEEtran.bst: No hyphenation pattern has been}%
\typeout{** loaded for the language `#1'. Using the pattern for}%
\typeout{** the default language instead.}%
\else
\language=\csname l@#1\endcsname
\fi
#2}}
\providecommand{\BIBdecl}{\relax}
\BIBdecl

\bibitem{oord2016wavenet}
A.~v.~d. Oord, S.~Dieleman, H.~Zen, K.~Simonyan, O.~Vinyals, A.~Graves, N.~Kalchbrenner, A.~Senior, and K.~Kavukcuoglu, ``Wavenet: A generative model for raw audio,'' \emph{arXiv preprint arXiv:1609.03499}, 2016.

\bibitem{bai2018tcn}
S.~Bai, J.~Z. Kolter, and V.~Koltun, ``An empirical evaluation of generic convolutional and recurrent networks for sequence modeling,'' \emph{arXiv preprint arXiv:1803.01271}, 2018.

\bibitem{zhou2021informer}
H.~Zhou, S.~Zhang, J.~Peng, S.~Zhang, J.~Li, H.~Xiong, and W.~Zhang, ``Informer: Beyond efficient transformer for long sequence time-series forecasting,'' in \emph{Proceedings of the AAAI Conference on Artificial Intelligence}, vol.~35, no.~12, 2021, pp. 11\,106--11\,115.

\bibitem{cont2001empirical}
R.~Cont, ``Empirical properties of asset returns: stylized facts and statistical issues,'' \emph{Quantitative Finance}, vol.~1, no.~2, pp. 223--236, 2001.

\bibitem{calvet2002multifractality}
L.~E. Calvet and A.~J. Fisher, ``Multifractality in asset returns: theory and evidence,'' \emph{Review of Economics and Statistics}, vol.~84, no.~3, pp. 381--406, 2002.

\bibitem{koscielny1998long}
E.~Koscielny-Bunde, A.~Bunde, S.~Havlin, H.~E. Roman, Y.~Goldreich, and H.-J. Schellnhuber, ``Indication of a universal persistence law governing atmospheric variability,'' \emph{Physical Review Letters}, vol.~81, no.~3, pp. 729--732, 1998.

\bibitem{frisch1995turbulence}
U.~Frisch, \emph{Turbulence: the legacy of {AN} {Kolmogorov}}.\hskip 1em plus 0.5em minus 0.4em\relax Cambridge University Press, 1995.

\bibitem{he2014scale}
B.~J. He, ``Scale-free brain activity: past, present, and future,'' \emph{Trends in Cognitive Sciences}, vol.~18, no.~9, pp. 480--487, 2014.

\bibitem{zilber2013learning}
N.~Zilber, P.~Ciuciu, P.~Abry, and V.~van Wassenhove, ``Learning-induced modulation of scale-free properties of brain activity measured with {MEG},'' in \emph{IEEE International Symposium on Biomedical Imaging (ISBI)}, 2013.

\bibitem{leland1994self}
W.~E. Leland, M.~S. Taqqu, W.~Willinger, and D.~V. Wilson, ``On the self-similar nature of {Ethernet} traffic (extended version),'' \emph{IEEE/ACM Transactions on Networking}, vol.~2, no.~1, pp. 1--15, 1994.

\bibitem{mandelbrot1968fbm}
B.~B. Mandelbrot and J.~W. Van~Ness, ``Fractional {B}rownian motions, fractional noises and applications,'' \emph{SIAM Review}, vol.~10, no.~4, pp. 422--437, 1968.

\bibitem{papamakarios2017masked}
G.~Papamakarios, T.~Pavlakou, and I.~Murray, ``Masked autoregressive flow for density estimation,'' in \emph{NeurIPS}, 2017.

\bibitem{dinh2017realnvp}
L.~Dinh, J.~Sohl-Dickstein, and S.~Bengio, ``Density estimation using {Real NVP},'' in \emph{ICLR}, 2017.

\bibitem{papamakarios2019normalizing}
G.~Papamakarios, E.~Nalisnick, D.~J. Rezende, S.~Mohamed, and B.~Lakshminarayanan, ``Normalizing flows for probabilistic modeling and inference,'' \emph{JMLR}, vol.~22, no.~57, pp. 1--64, 2021.

\bibitem{cohen2016group}
T.~Cohen and M.~Welling, ``Group equivariant convolutional networks,'' in \emph{International Conference on Machine Learning (ICML)}, 2016, pp. 2990--2999.

\bibitem{cohen2018spherical}
T.~S. Cohen, M.~Geiger, J.~K{\"o}hler, and M.~Welling, ``Spherical {CNN}s,'' in \emph{ICLR}, 2018.

\bibitem{weiler2019general}
M.~Weiler and G.~Cesa, ``General {E}(2)-equivariant steerable {CNNs},'' in \emph{NeurIPS}, 2019.

\bibitem{worrall2019deep}
D.~E. Worrall and M.~Welling, ``Deep scale-spaces: equivariance over scale,'' in \emph{NeurIPS}, 2019.

\bibitem{sosnovik2020scale}
I.~Sosnovik, M.~Szmaja, and A.~Smeulders, ``Scale-equivariant steerable networks,'' in \emph{ICLR}, 2020.

\bibitem{zhu2022scaling}
W.~Zhu, Q.~Qiu, R.~Calderbank, G.~Sapiro, and X.~Cheng, ``Scaling-translation-equivariant networks with decomposed convolutional filters,'' \emph{Journal of Machine Learning Research}, vol.~23, no.~68, pp. 1--45, 2022.

\bibitem{bruna2013scattering}
J.~Bruna and S.~Mallat, ``Invariant scattering convolution networks,'' \emph{IEEE TPAMI}, vol.~35, no.~8, pp. 1872--1886, 2013.

\bibitem{andén2014deep}
J.~And{\'e}n and S.~Mallat, ``Deep scattering spectrum,'' \emph{IEEE Transactions on Signal Processing}, vol.~62, no.~16, pp. 4114--4128, 2014.

\bibitem{oyallon2017scaling}
E.~Oyallon, E.~Belilovsky, and S.~Zagoruyko, ``Scaling the scattering transform: deep hybrid networks,'' in \emph{ICCV}, 2017.

\bibitem{andreux2020kymatio}
M.~Andreux, T.~Angles, G.~Exarchakis, R.~Leonarduzzi, G.~Rochette, L.~Thiry, J.~Zarka, S.~Mallat, J.~And{\'e}n, E.~Belilovsky \emph{et~al.}, ``Kymatio: scattering transforms in {Python},'' \emph{JMLR}, vol.~21, no.~60, pp. 1--6, 2020.

\bibitem{salinas2020deepar}
D.~Salinas, V.~Flunkert, J.~Gasthaus, and T.~Januschowski, ``{DeepAR}: probabilistic forecasting with autoregressive recurrent networks,'' \emph{International Journal of Forecasting}, vol.~36, no.~3, pp. 1181--1191, 2020.

\bibitem{rasul2021autoregressive}
K.~Rasul, C.~Seward, I.~Schuster, and R.~Vollgraf, ``Autoregressive denoising diffusion models for multivariate probabilistic time series forecasting,'' in \emph{ICML}, 2021, pp. 8857--8868.

\bibitem{voelker2019legendre}
A.~Voelker, I.~Kaji{\'c}, and C.~Eliasmith, ``Legendre memory units: continuous-time representation in recurrent neural networks,'' in \emph{NeurIPS}, 2019.

\bibitem{gu2022s4}
A.~Gu, K.~Goel, and C.~R{\'e}, ``Efficiently modeling long sequences with structured state spaces,'' in \emph{ICLR}, 2022.

\bibitem{decreusefond2000fbm}
L.~Decreusefond and A.~S. {\"U}st{\"u}nel, ``Stochastic analysis of the fractional {Brownian} motion,'' \emph{Potential Analysis}, vol.~10, no.~2, pp. 177--214, 1999.

\bibitem{mandelbrot1997msm}
B.~B. Mandelbrot, A.~J. Fisher, and L.~E. Calvet, ``A multifractal model of asset returns,'' \emph{Cowles Foundation Discussion Paper}, no. 1164, 1997.

\bibitem{ivanov1999multifractality}
P.~C. Ivanov, L.~A.~N. Amaral, A.~L. Goldberger, S.~Havlin, M.~G. Rosenblum, Z.~R. Struzik, and H.~E. Stanley, ``Multifractality in human heartbeat dynamics,'' \emph{Nature}, vol. 399, no. 6735, pp. 461--465, 1999.

\bibitem{liu1999statistical}
Y.~Liu, P.~Gopikrishnan, P.~Cizeau, M.~Meyer, C.-K. Peng, and H.~E. Stanley, ``Statistical properties of the volatility of price fluctuations,'' \emph{Physical Review E}, vol.~60, no.~2, pp. 1390--1400, 1999.

\bibitem{turcotte1997fractals}
D.~L. Turcotte, \emph{Fractals and chaos in geology and geophysics}.\hskip 1em plus 0.5em minus 0.4em\relax Cambridge University Press, 1997.

\bibitem{durall2020watch}
R.~Durall, M.~Keuper, and J.~Keuper, ``Watch your up-convolution: {CNN} based generative deep neural networks are failing to reproduce spectral distributions,'' in \emph{CVPR}, 2020.

\bibitem{khayatkhoei2022spatial}
M.~Khayatkhoei and A.~Elgammal, ``Spatial frequency bias in convolutional generative adversarial networks,'' in \emph{AAAI}, 2022.

\bibitem{daubechies1988orthonormal}
I.~Daubechies, ``Orthonormal bases of compactly supported wavelets,'' \emph{Communications on Pure and Applied Mathematics}, vol.~41, no.~7, pp. 909--996, 1988.

\bibitem{perez2018film}
E.~Perez, F.~Strub, H.~De~Vries, V.~Dumoulin, and A.~Courville, ``{FiLM}: visual reasoning with a general conditioning layer,'' in \emph{AAAI}, 2018.

\bibitem{welch1967use}
P.~Welch, ``The use of fast {F}ourier transform for the estimation of power spectra,'' \emph{IEEE Transactions on Audio and Electroacoustics}, vol.~15, no.~2, pp. 70--73, 1967.

\bibitem{bollerslev1986garch}
T.~Bollerslev, ``Generalized autoregressive conditional heteroskedasticity,'' \emph{Journal of Econometrics}, vol.~31, no.~3, pp. 307--327, 1986.

\bibitem{rizzo2016energy}
M.~L. Rizzo and G.~J. Sz{\'e}kely, ``Energy distance,'' \emph{Wiley Interdisciplinary Reviews: Computational Statistics}, vol.~8, no.~1, pp. 27--38, 2016.

\bibitem{willinger1997scaling}
W.~Willinger, M.~S. Taqqu, R.~Sherman, and D.~V. Wilson, ``Self-similarity through high-variability: statistical analysis of {Ethernet LAN} traffic at the source level,'' \emph{IEEE/ACM Transactions on Networking}, vol.~5, no.~1, pp. 71--86, 1997.

\bibitem{voss1975fractal}
R.~F. Voss and J.~Clarke, ``'1/f' noise in music and speech,'' \emph{Nature}, vol. 258, no. 5533, pp. 317--318, 1975.

\bibitem{wendt2007multifractality}
H.~Wendt, P.~Abry, and S.~Jaffard, ``Bootstrap for empirical multifractal analysis,'' \emph{IEEE Signal Processing Magazine}, vol.~24, no.~4, pp. 38--48, 2007.

\end{thebibliography}
\end{document}